\documentclass[11pt]{article}
\usepackage[margin=0.82in]{geometry}
\usepackage{amsmath,amssymb,amsfonts,bm,mathtools}
\usepackage{graphicx,booktabs,array,microtype,xcolor,enumitem,float}
\usepackage{amsthm}
\usepackage{hyperref}
\usepackage[capitalise,noabbrev]{cleveref}
\hypersetup{colorlinks=true,linkcolor=blue!45!black,citecolor=blue!45!black,urlcolor=blue!45!black}
\newcommand{\E}{\mathbb{E}}
\newcommand{\R}{\mathbb{R}}
\newcommand{\diag}{\operatorname{diag}}

\newcommand{\Norm}{\mathcal{N}}
\newcommand{\WP}{W_{\!P}}
\newcommand{\uepi}{u_{\mathrm{epi}}}
\newcommand{\ualea}{u_{\mathrm{alea}}}

\newcommand{\Levi}{\mathcal{L}_{\mathrm{evi}}}

\newtheorem{theorem}{Theorem}

\theoremstyle{definition}
\newtheorem{definition}{Definition}

\title{\textbf{ELECTRIC: Evidential Learning--Enhanced CT Reconstruction via Iterative Correction}\\[0.35em]
\large -- Toward Adaptive Bayesian Reconstruction via Learned Prior Precision}
\author{Ge Wang\\
\small Department of Biomedical Engineering\\
\small Department of Electrical, Computer and Systems Engineering\\
\small Department of Computer Science\\
\small Rensselaer Polytechnic Institute, Troy, New York, USA\\
wangg6@rpi.edu
\date{\today}}
\begin{document}
\maketitle

\begin{abstract}
Computed tomography (CT) reconstruction is fundamentally a Bayesian inverse problem in which measurement information is combined with prior knowledge. Classical statistical reconstruction usually assumes that confidence in the prior is prespecified and fixed. Modern deep learning methods learn increasingly expressive image priors and may estimate predictive uncertainty, but they rarely use that uncertainty to determine how strongly the learned prior should influence reconstruction. Here we introduce \textbf{ELECTRIC} (Evidential Learning--Enhanced CT Reconstruction via Iterative Correction), a physics-guided Bayesian formulation. An evidential neural network provides an image proposal and an error-predictive epistemic-uncertainty surrogate. The latter is converted into an adaptive precision field and inserted into a Poisson-weighted MAP update. The resulting image--evidence--precision--reconstruction loop treats prior confidence as a learned state variable of iterative reconstruction. In addition to the formulation and theoretical analysis, we report two simulation studies on image slices from the AAPM Mayo Clinic Low-Dose CT dataset: a mechanism-validation pilot using transparent surrogate estimators, and a feasibility study in which a trained Normal--Inverse--Gamma evidential network drives the full closed loop. On held-out patients, the learned prior mean reduces reconstruction error by roughly 70\% relative to filtered back-projection, the learned epistemic uncertainty is error-predictive and supports selective trust, and the physics-guided update restores measurement consistency while the adaptive-precision reconstruction matches or exceeds a validation-tuned fixed prior and remains markedly more robust to prior-strength misspecification. Together these results demonstrate the complete ELECTRIC closed-loop pipeline, while identifying formal uncertainty calibration and joint training as the principal directions for future work.
\end{abstract}


\section{Introduction}
Image reconstruction is an inverse problem that balances two complementary information sources: \emph{measurement consistency}, dictated by imaging physics, and \emph{prior knowledge}, describing plausible image structure. In Bayesian terms, these are represented by a likelihood and a prior. The prior itself contains two conceptually distinct quantities: its mean, which describes what image is expected, and its precision, which determines how strongly that expectation should influence reconstruction.

Computed tomography (CT) provides a representative setting. Reconstruction methods have progressed from analytical inversion to statistical iterative reconstruction (SIR), model-based iterative reconstruction (MBIR), sparsity-promoting optimization, plug-and-play priors, regularization by denoising, deep unrolling, flow-based reconstruction, and foundation models. Much of this progress has focused on improving the prior mean: handcrafted models have been replaced by increasingly expressive learned representations. Comparatively less attention has been paid to learning the confidence assigned to those priors and, needless to say, to utilization of the confidence in image reconstruction.

Classical Bayesian reconstruction typically prescribes prior confidence through one or more regularization parameters. Hierarchical Bayesian methods improve this formulation by estimating covariance or precision hyperparameters inside a specified prior family. Deep learning further improves the prior representation, but commonly leaves its reliability fixed, implicit, or spatially uniform. This assumption is questionable across anatomy, pathology, scanner vendors, dose levels, acquisition geometries, and distribution shifts.

Bayesian neural networks, Monte Carlo dropout, deep ensembles, Laplace approximations, and evidential learning can estimate predictive uncertainty. Such estimates are useful for visualization, calibration, quality assurance, and failure detection. Yet uncertainty generally remains passive: it is reported after reconstruction and not intended to alter the reconstruction result.

Here we propose \textbf{ELECTRIC} for Evidential Learning–Enhanced CT Reconstruction via Iterative Correction, a physics-guided Bayesian formulation in which prior precision is a learned, spatially adaptive state variable of iterative reconstruction. Instead of prescribing a global prior strength, ELECTRIC estimates voxel-wise evidence, converts epistemic uncertainty into an adaptive precision field, and uses that field in a MAP update. Strongly supported regions receive greater prior influence; uncertain regions rely more heavily on the measured projection data. The result is a closed loop: image reconstruction $\rightarrow$ evidence $\rightarrow$ uncertainty $\rightarrow$ precision $\rightarrow$ updated image.

The central contribution is therefore not a particular NIG mapping or network architecture. It is the proposal that learned priors have two quantities to estimate: a mean and a confidence. Current deep reconstruction has largely learned what to reconstruct. ELECTRIC proposes learning how much that reconstruction should be trusted. Evidential regression is adopted in the present work as one computationally efficient realization, whereas Bayesian neural networks, deep ensembles, diffusion-based uncertainty estimators, and future probabilistic models can all be incorporated within the proposed framework without changing its underlying Bayesian principle.
In other words, conventional deep reconstruction primarily learns
what image to reconstruct, whereas ELECTRIC proposes learning both
what image to reconstruct and a precision or confidence map associated with that image.

The paper contributes:
\begin{enumerate}[leftmargin=1.5em]
\item a Bayesian formulation in which prior precision is learned rather than prescribed;
\item a principled separation between measurement uncertainty in the likelihood and model uncertainty in the prior; and
\item a multi-stage training strategy aligned with the intermediate states encountered during iterative inference.
\end{enumerate}

\begin{center}
\fbox{\parbox{0.91\textwidth}{\textbf{Central thesis of ELECTRIC.} Classical Bayesian reconstruction estimates the image while assuming prior confidence is known. ELECTRIC estimates both the learned prior mean and how strongly that prior should influence the current reconstruction:
\[
(\gamma,\nu,\alpha,\beta)\longrightarrow \uepi\longrightarrow \Sigma_{\mathrm{prior}}\longrightarrow \Lambda\longrightarrow \text{physics-guided MAP reconstruction}.
\]}}
\end{center}

\section{Bayesian Reconstruction: Foundations and Evolution}
This section reviews the concepts needed to distinguish a learned prior mean from learned prior confidence and places ELECTRIC within the evolution from fixed regularization to adaptive Bayesian precision.

\begin{table}[t]
\centering
\caption{Major variables and symbols used in ELECTRIC.}
\label{tab:notation}
\small
\begin{tabular}{>{\raggedright\arraybackslash}p{0.18\textwidth} p{0.73\textwidth}}
\toprule
\textbf{Symbol} & \textbf{Meaning}\\
\midrule
$x$ & CT attenuation image.\\
$p$ & Log-transformed projection measurement.\\
$A$, $A^T$ & Forward projector and adjoint/backprojector.\\
$\Sigma_{\mathrm{meas}}$, $\WP$ & Measurement covariance and precision, with $\WP=\Sigma_{\mathrm{meas}}^{-1}$.\\
$\gamma$ & Learned prior mean or image proposal.\\
$\nu,\alpha,\beta$ & NIG evidence parameters, with $\nu>0$, $\alpha>1$, and $\beta>0$ when the uncertainty moments below are used.\\
$\ualea$, $\uepi$ & Aleatoric and epistemic uncertainty surrogates.\\
$\Sigma_{\mathrm{prior}}$, $\Lambda$ & Adaptive prior covariance and Bayesian prior precision, with $\Lambda=\Sigma_{\mathrm{prior}}^{-1}$; precision is interpreted as prior confidence.\\
$\lambda_{\min},\lambda_{\max}$
&
Minimum and maximum allowable prior precision, respectively.
The maximum is attained at zero epistemic uncertainty, whereas the
minimum provides a strictly positive precision floor.\\
$\kappa$ & Sensitivity coefficient; $\kappa\uepi$ is dimensionless.\\
$x^{(k)},\gamma^{(k)},\Lambda^{(k)}$ & Image, prior mean, and prior precision at outer iteration $k$.\\
$x_{\mathrm{gt}}$ & Ground-truth/reference image used during supervised training.\\
\bottomrule
\end{tabular}
\end{table}

\subsection{Statistical Bayesian CT reconstruction}
Let $x\in\R^N$ denote the attenuation image and $p\in\R^M$ the log-transformed projection data. Under the standard local Gaussian approximation to transformed Poisson measurements,
\begin{equation}
p\mid x\sim\Norm(Ax,\Sigma_{\mathrm{meas}}),
\end{equation}
and
\begin{equation}
\mathcal L_{\mathrm{data}}(x)=\frac12(Ax-p)^T\WP(Ax-p),\qquad \WP=\Sigma_{\mathrm{meas}}^{-1}.
\end{equation}
Precision is inverse covariance: lower measurement variance produces greater statistical weight. With a Gaussian prior $x\sim\Norm(\gamma,\lambda^{-1}I)$, conventional MAP reconstruction solves
\begin{equation}
x^*=\arg\min_x\frac12\|Ax-p\|_{\WP}^2+\frac{\lambda}{2}\|x-\gamma\|_2^2.
\end{equation}
The mean $\gamma$ specifies the expected image; the scalar precision $\lambda$ specifies how strongly that expectation is enforced.

\subsection{Bayesian deep learning and uncertainty estimation}
Aleatoric uncertainty represents irreducible variability in the observations, whereas epistemic uncertainty reflects limited model knowledge. Bayesian neural networks, Monte Carlo dropout, deep ensembles, and Laplace approximations provide different approximations to predictive uncertainty, often at the cost of repeated forward passes or additional optimization. Most uncertainty-aware reconstruction methods use these estimates for calibration, visualization, or failure detection. ELECTRIC differs functionally: uncertainty changes the reconstruction objective itself.

\subsection{Evidential regression and parameter constraints}
Deep evidential regression predicts a distribution over a Gaussian mean and variance. For a scalar target $y$,
\begin{align}
y\mid\mu,\sigma^2 &\sim \Norm(\mu,\sigma^2),\\
\mu\mid\sigma^2 &\sim \Norm\!\left(\gamma,\frac{\sigma^2}{\nu}\right),\\
\sigma^2 &\sim \operatorname{InvGamma}(\alpha,\beta),
\end{align}
with $\nu>0$, $\alpha>0$, and $\beta>0$. 
%Marginalization yields a Student-$t$ distribution with $2\alpha$ degrees of freedom. 
The uncertainty moments used by ELECTRIC require $\alpha>1$:
\begin{align}
\E[y]&=\gamma,\\
\ualea&=\frac{\beta}{\alpha-1},\\
\uepi&=\frac{\beta}{\nu(\alpha-1)},\\
\operatorname{Var}(y)&=\frac{\beta(1+\nu)}{\nu(\alpha-1)}.
\end{align}
%No condition such as $\nu>2$ is required; $\nu$ must only remain positive.
The three quantities above have distinct physical interpretations.
The aleatoric uncertainty $u_{\mathrm{alea}}$ measures the expected
observation noise or intrinsic variability that cannot be reduced even
with additional training data.
In contrast, the epistemic uncertainty
$u_{\mathrm{epi}}$ quantifies uncertainty arising from limited model
knowledge or insufficient evidence and therefore may decrease as more
representative training data or a stronger predictive model become
available.
The predictive variance combines both effects,
$\mathrm{Var}(y)
=
u_{\mathrm{alea}}+u_{\mathrm{epi}},$
indicating that the overall uncertainty consists of an irreducible
measurement component and a reducible model component.
Accordingly, the predictive variance itself is not used directly
to regulate adaptive prior precision, because it combines both
reducible and irreducible sources of uncertainty.
Within the proposed ELECTRIC framework, only the epistemic component is
used to regulate Bayesian prior precision, whereas aleatoric uncertainty
is already represented by the measurement likelihood through the data
precision matrix $W_P$.

%Recent critical analyses, including the deliberately titled ``unreasonable effectiveness'' study, argue that 
These evidential outputs may behave as useful heuristics without constituting exact Bayesian posterior uncertainties. ELECTRIC adopts this interpretation. The predicted $\uepi$ is an operational surrogate for prior reliability whose adequacy must be tested by calibration and error analysis.

\subsection{Hierarchical Bayesian reconstruction, IAS, and ARD}
Adaptive prior strength has a long history in Bayesian inverse problems. Hierarchical models augment the image with hyperparameters controlling covariance or component-wise variance. IAS algorithms alternate image and hyperparameter updates, while ARD and sparse Bayesian learning estimate relevance precisions through evidence maximization. These methods show that estimating precision is not itself new.

ELECTRIC differs in the mechanism and role of adaptation. Classical hypermodels infer parameters inside a prescribed prior family. ELECTRIC predicts a state-dependent spatial precision from the current reconstruction and measured data, then repeatedly reinjects that precision into a physics-based MAP update. It is therefore complementary to hierarchical Bayesian inference, with the context-dependent prior reliability that is learned from data and updated across reconstruction states.

\subsection{Boundary with adaptive regularization}
Spatially adaptive TV, anisotropic penalties, local sparsity weights, and other similar methods also vary regularization strength. Their weights are commonly derived from handcrafted features, risk estimates, or fixed local models. ELECTRIC instead learns a Bayesian precision field from an uncertainty estimator and places it explicitly in a prior covariance model. %The conceptual progression is
%\[
%\text{fixed precision}\rightarrow\text{inferred hyperparameters}\rightarrow\text{learned prior mean}\rightarrow\textbf{learned prior precision}.
%\]

\section{Adaptive Bayesian Prior Precision}
The formulation begins with Bayesian precision 
rather than with a specific uncertainty estimator. The prior mean specifies what image is expected; the prior precision specifies how strongly that expectation should influence reconstruction. ELECTRIC proposes learning both.

\subsection{Conceptual overview}
Figure~\ref{fig:concept} summarizes the transition from fixed prior precision, through passive uncertainty reporting, to active adaptive precision. The central state variable is not uncertainty itself but the Bayesian precision constructed from it.

\begin{figure}[t]
\centering
\includegraphics[width=\textwidth]{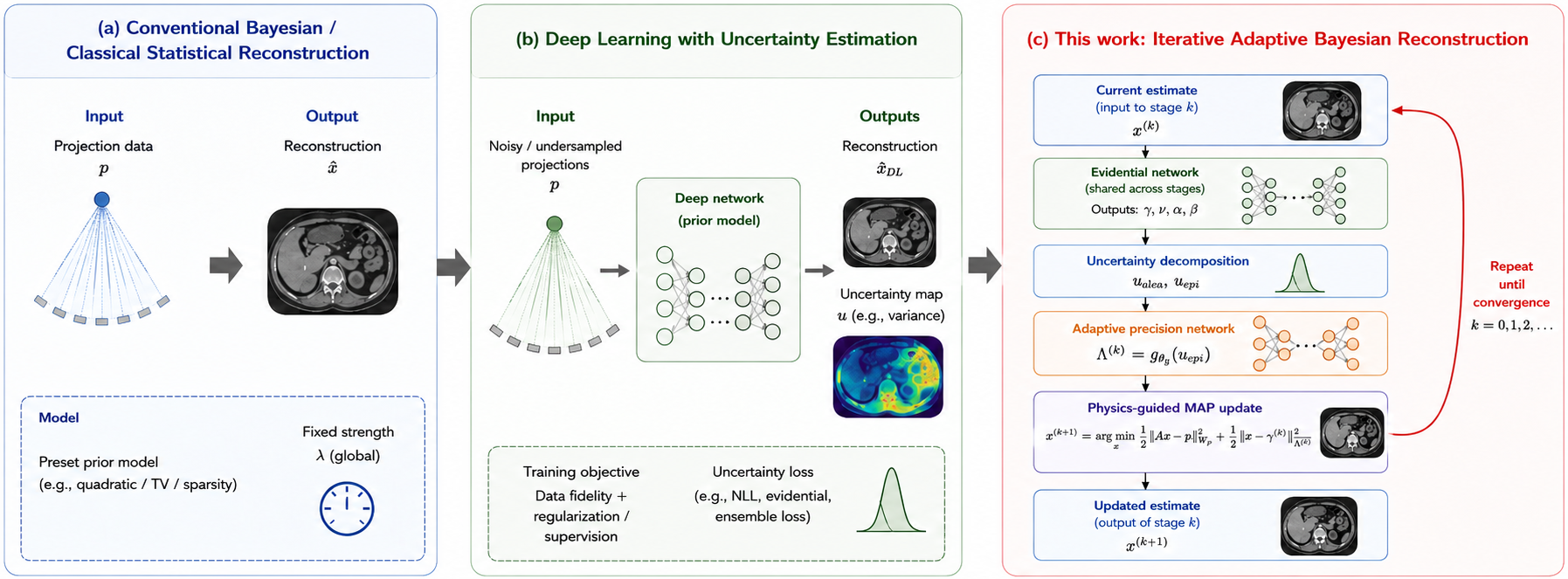}
\caption{Conceptual evolution toward ELECTRIC. (a) Classical methods use a fixed, manually prescribed prior confidence (regularization strength). (b) Deep networks can estimate uncertainty but typically do not use it in reconstruction. (c) ELECTRIC (this work) converts estimated epistemic uncertainty into an adaptive prior precision field and incorporates it into a physics-guided MAP update in an iterative loop.}
\label{fig:concept}
\end{figure}

\subsection{Definitions}
\begin{definition}[Learned prior mean]
At outer iteration $k$, the network output $\gamma^{(k)}$ is the learned prior mean: the image toward which the Bayesian prior is centered.
\end{definition}

\begin{definition}[Bayesian prior precision]
The Bayesian prior precision
\begin{equation}
\Lambda^{(k)}=\left(\Sigma_{\mathrm{prior}}^{(k)}\right)^{-1}
\end{equation}
is the mathematical quantity that determines how strongly the learned prior mean influences the MAP estimate. Throughout this paper, the term \emph{prior confidence} is used as an intuitive interpretation of Bayesian prior precision. Thus, a larger prior precision corresponds to greater confidence in the learned prior and stronger influence on the reconstruction, whereas a smaller prior precision corresponds to lower confidence and greater reliance on the measured projection data.
\end{definition}

\begin{definition}[Adaptive Bayesian prior]
A prior is adaptive when its precision depends on the current reconstruction state and measurements,
\begin{equation}
\Lambda^{(k)}=\Lambda_\theta(x^{(k)},p).
\end{equation}
\end{definition}

\subsection{Bayesian design principles}
The formulation follows five design principles: (i) measurement uncertainty belongs to the data likelihood; (ii) model uncertainty belongs to the image prior; (iii) prior precision should not increase with epistemic uncertainty; (iv) prior precision must remain positive and bounded to stabilize each subproblem; and (v) conventional MAP should be recovered when epistemic uncertainty vanishes.

The measurement model contributes
\begin{equation}
\mathcal L_{\mathrm{data}}(x)=\frac12(Ax-p)^T\WP(Ax-p).
\end{equation}
Given a learned prior mean $\gamma^{(k)}$ and covariance $\Sigma_{\mathrm{prior}}^{(k)}$, the adaptive MAP objective is
\begin{equation}
\boxed{\mathcal J_k(x)=\frac12(Ax-p)^T\WP(Ax-p)+\frac12(x-\gamma^{(k)})^T\Lambda^{(k)}(x-\gamma^{(k)})}.
\label{eq:map}
\end{equation}
The Bayesian construction is therefore
\begin{equation}
\Sigma_{\mathrm{prior}}^{(k)}
=
h\!\left(u_{\mathrm{epi}}^{(k)}\right),
\qquad
\Lambda^{(k)}
=
g\!\left(u_{\mathrm{epi}}^{(k)}\right),
\label{eq:uncertainty_maps}
\end{equation}
where $h$ denotes an uncertainty-to-covariance mapping and
$g$ denotes the corresponding uncertainty-to-precision mapping.

For diagonal covariance models,
\[
\Lambda^{(k)}
=
\operatorname{diag}
\!\left(
h(u_{\rm epi}^{(k)})
\right)^{-1},
\]
where the inverse is taken elementwise.
Different analytical or learned mappings may be adopted without
changing the overall ELECTRIC framework.

\subsection{One evidential realization}
At iteration $k$, the evidential network predicts
\begin{equation}
(\gamma^{(k)},\nu^{(k)},\alpha^{(k)},\beta^{(k)})=f_\theta(x^{(k)},p),
\end{equation}
and
\begin{equation}
\uepi^{(k)}=\frac{\beta^{(k)}}{\nu^{(k)}(\alpha^{(k)}-1)}.
\end{equation}
Only epistemic uncertainty regulates the prior. Aleatoric measurement variability is already represented by $\WP$ in the data likelihood. Epistemic uncertainty measures limited reliability of the learned prior mean and therefore belongs naturally to the prior.

\subsection{Bayesian-consistent evidential parameterization}

The evidential realization in Eq.~(16) directly predicts the four
Normal--Inverse--Gamma (NIG) parameters

\[
(\gamma,\nu,\alpha,\beta)=f_\theta(x,p).
\]

This parameterization is computationally convenient and has been widely
adopted in deep evidential regression.
However, the predicted NIG parameters should not necessarily be interpreted
as the exact posterior resulting from a Bayesian conjugate update.
Instead, they are operational evidential quantities whose usefulness is
ultimately validated through calibration, uncertainty--error correlation,
and downstream reconstruction performance.

An alternative Bayesian-consistent realization may be constructed by
predicting sufficient statistics of a virtual local dataset rather than
predicting the four NIG parameters directly.
Specifically, the network predicts

\[
(n,\bar y,S),
\]

where $n$ denotes the effective evidence size,
$\bar y$ the virtual sample mean,
and $S$ the virtual sample scatter.
Starting from an initial NIG prior

\[
(\gamma_0,\nu_0,\alpha_0,\beta_0),
\]

the posterior parameters are obtained analytically through the conjugate
Normal--Inverse--Gamma update

\begin{align}
\nu &= \nu_0+n,\\
\gamma &= \frac{\nu_0\gamma_0+n\bar y}{\nu_0+n},\\
\alpha &= \alpha_0+\frac{n}{2},\\
\beta &= \beta_0+\frac{S}{2}
+\frac{\nu_0 n}{2(\nu_0+n)}
(\bar y-\gamma_0)^2.
\end{align}

A natural choice in iterative reconstruction is to use the current
reconstruction as the prior mean,

\[
\gamma_0=x^{(k)},
\]

while employing globally shared weak prior hyperparameters
$(\nu_0,\alpha_0,\beta_0)$.
The network then predicts only the incremental evidence
$(n,\bar y,S)$,
which is converted into NIG parameters through exact Bayesian updating.

With $n > 0$ and $S\geq 0$, this realization places every
predicted NIG parameter set in the continuously relaxed reachable
set by construction. Integer-valued $n$ recovers exact finite-sample
Bayesian reachability.
The resulting uncertainty estimation remains fully compatible with the
overall ELECTRIC framework, because the subsequent uncertainty-to-precision
mapping and MAP reconstruction are unchanged.
The direct parameterization in Eq.~(16) is adopted in the present work
because of its simplicity and computational efficiency, whereas the
Bayesian-consistent parameterization above provides an attractive direction
for future investigation.

\subsection{Uncertainty-to-Precision Mapping}

Many direct uncertainty-to-precision mappings $g$ satisfy these
design principles. A bounded rational realization with a strictly
positive precision floor is
\begin{equation}
\Lambda^{(k)}
=
\operatorname{diag}
\left[
\lambda_{\min}
+
\frac{\lambda_{\max}-\lambda_{\min}}
{1+\kappa u_{\mathrm{epi}}^{(k)}}
\right],
\label{eq:adaptive_precision}
\end{equation}
where
\[
0<\lambda_{\min}<\lambda_{\max},
\]
and $\kappa>0$ has reciprocal units of
$u_{\mathrm{epi}}^{(k)}$, so that
$\kappa u_{\mathrm{epi}}^{(k)}$ is dimensionless.
Here, $\lambda_{\max}$ is the precision attained at zero epistemic
uncertainty, whereas $\lambda_{\min}$ is a positive precision floor
that preserves regularization in highly uncertain or poorly observed
directions.

Eq.~(22) is not asserted to be uniquely optimal.
It is a simple bounded rational map satisfying positivity,
smoothness, boundedness, monotone decrease, recovery of
$\lambda_{\max}I$ at zero epistemic uncertainty, and convergence
toward the positive precision floor $\lambda_{\min}I$ as uncertainty
increases. Exponential, logistic, spline-based, or learned
uncertainty-to-precision mappings $g$ remain compatible with the
overall ELECTRIC formulation.

More generally, the uncertainty-to-precision map need not be prescribed analytically.
Instead, it may be parameterized by a trainable function

\[
\Lambda^{(k)} = g_{\phi}\!\left(u_{\mathrm{epi}}^{(k)}\right),
\]

where $\phi$ denotes learnable parameters.
Such a realization allows the mapping from uncertainty to Bayesian prior
precision to be optimized jointly with the reconstruction network under the
overall imaging objective.
Rather than relying on a predetermined functional relationship,
the model can learn how statistical uncertainty can be translated
into task-aware Bayesian prior confidence so as to improve
reconstruction performance while preserving the Bayesian role of
adaptive precision.
The bounded rational map in Eq.~(22) should therefore be viewed as one
interpretable realization of the more general ELECTRIC framework, whereas
learned mappings, spline parameterizations, monotone neural networks,
or other differentiable architectures remain fully compatible with the
proposed formulation.
When monotone dependence is desired, $g_\phi$ should be
parameterized as a monotone decreasing network or constrained
spline so that greater epistemic uncertainty cannot produce
greater prior precision.
This constraint preserves the Bayesian interpretation that
regions with lower confidence in the learned prior should
not receive stronger prior influence during reconstruction.

\subsection{Multi-stage training matched to inference}
The initialization can be, for example, 
\begin{equation}
x^{(0)}=\operatorname{FBP}(p).
\end{equation}
A training trajectory recursively applies one evidence-estimation and MAP-update cycle,
\begin{equation}
x_i^{(k+1)}=\mathcal M_\theta(x_i^{(k)},p_i).
\end{equation}
The total loss at stage $k$ is as follows:
\begin{equation}
L_{\mathrm{stage}}^{(k)}
=
\lambda_{\mathrm{rec}}L_{\mathrm{rec}}^{(k)}
+
\lambda_{\mathrm{NIG}}L_{\mathrm{NIG}}^{(k)}
+
\lambda_{\mathrm{evi}}L_{\mathrm{evi}}^{(k)}
+
\lambda_{\mathrm{sm}}L_{\mathrm{sm}}^{(k)} .
\end{equation}
The reconstruction-consistency term is applied to the MAP output,
\begin{equation}
L_{\mathrm{rec}}^{(k)}
=
\frac12
\left\|
Ax_i^{(k+1)}
-
p_i
\right\|_{W_{P,i}}^2
+
\rho
\left\|
x_i^{(k+1)}
-
x_{\mathrm{gt},i}
\right\|_1 .
\end{equation}
%so that training optimizes the final stage output rather than only the proposal $\gamma^{(k)}$. 
The Student-$t$ NLL and evidence regularizer directly supervise the evidential proposal and uncertainty parameters. Writing $e_r^{(k)}=x_{\mathrm{gt},r}-\gamma_r^{(k)}$ and $\Omega_r^{(k)}=2\beta_r^{(k)}(1+\nu_r^{(k)})$, one convenient NLL form is
\begin{equation}
\begin{aligned}
\ell_{\mathrm{NIG},r}^{(k)}={}&\frac12\log\!\left(\frac{\pi}{\nu_r^{(k)}}\right)-\alpha_r^{(k)}\log\Omega_r^{(k)}\\
&+\left(\alpha_r^{(k)}+\frac12\right)\log\!\left(\nu_r^{(k)}(e_r^{(k)})^2+\Omega_r^{(k)}\right)\\
&+\log\Gamma(\alpha_r^{(k)})-\log\Gamma\!\left(\alpha_r^{(k)}+\frac12\right).
\end{aligned}
\end{equation}
The evidence penalty may be written
\begin{equation}
\Levi^{(k)}=\frac1N\sum_r |e_r^{(k)}|(2\nu_r^{(k)}+\alpha_r^{(k)}).
\end{equation}

For fully unrolled training, the overall objective may be formed by
summing or weighting the stage losses across all iterations. In
stage-wise training, gradients pass through $\gamma^{(k)}$, $\Lambda^{(k)}$, and the MAP solution through CG. In stage-wise training, the completed state $x^{(k+1)}$ is detached before becoming the input to stage $k+1$; $\Lambda^{(k)}$ is not detached inside its own stage, so that the current reconstruction loss still trains the uncertainty-to-precision pathway. Optionally, the last term $\lambda_{\mathrm{sm}}\mathcal L_{\mathrm{sm}}^{(k)}$ regularizes the spatial fields of the evidential parameters, encouraging neighboring voxels to have similar parameter values. Any standard smoothness regularizer (e.g., squared spatial gradients or total variation) may be used.

\begin{figure}[t]
\centering
\includegraphics[width=\textwidth]{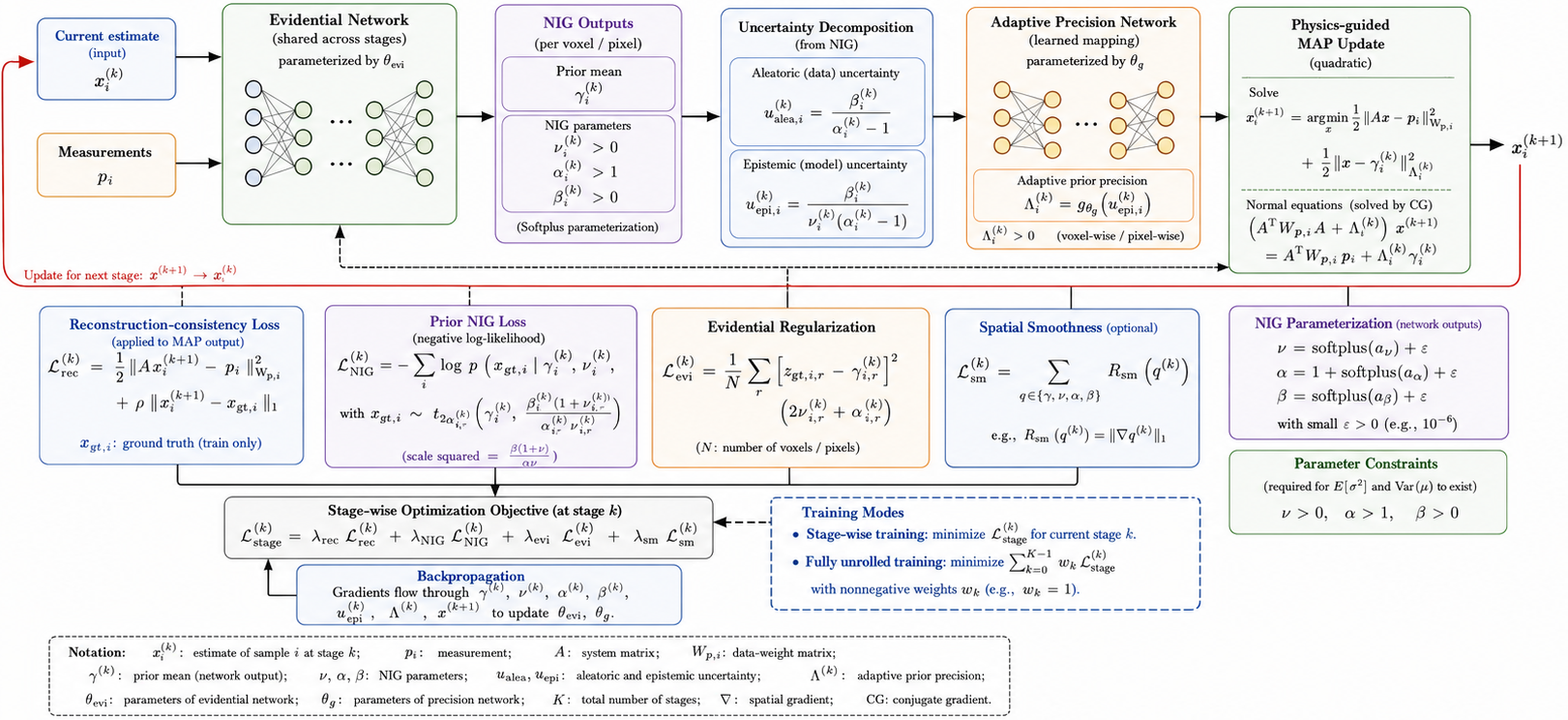}
\caption{Per-stage computation and optimization pipeline of ELECTRIC.
Training supervises both the learned prior proposal and the MAP-updated
reconstruction across intermediate states. Inference alternates
prior estimation, uncertainty-to-precision conversion, and a
physics-guided MAP update.}
\label{fig:train_infer}
\end{figure}

\subsection{MAP update and numerical solution}
For fixed $\gamma^{(k)}$ and $\Lambda^{(k)}$, the first-order condition of \eqref{eq:map} is
\begin{equation}
(A^T\WP A+\Lambda^{(k)})x^{(k+1)}=A^T\WP p+\Lambda^{(k)}\gamma^{(k)}.
\label{eq:normal}
\end{equation}
The coefficient matrix is symmetric positive definite whenever $\Lambda^{(k)}\succ0$, so the system can be solved by conjugate gradients. The previous outer iterate $x^{(k)}$ is used as a warm start because adjacent precision fields generally change gradually.

\subsection{General properties}
For convenience, define
\[
\lambda_{\mathrm{L}}=\lambda_{\min},
\qquad
\lambda_{\mathrm{U}}=\lambda_{\max},
\]
where the subscripts denote the lower and upper admissible
precision bounds, respectively.
\begin{theorem}[Admissible adaptive precision]
Suppose
\[
A^{T}W_{P}A \succeq 0
\]
and the adaptive precision satisfies
\[
0<
\lambda_{\mathrm{L}}I
\preceq
\Lambda^{(k)}
\preceq
\lambda_{\mathrm{U}}I.
\]
Then each MAP subproblem has a unique minimizer, and
\[
\lambda_{\mathrm{L}}
\le
\lambda_{\min}
\!\left(
A^{T}W_{P}A+\Lambda^{(k)}
\right),
\]
\[
\lambda_{\max}
\!\left(
A^{T}W_{P}A+\Lambda^{(k)}
\right)
\le
\lambda_{\max}
\!\left(
A^{T}W_{P}A
\right)
+
\lambda_{\mathrm{U}}.
\]
Consequently,
\begin{equation}
\operatorname{cond}
\!\left(
A^{T}W_{P}A+\Lambda^{(k)}
\right)
\le
\frac{
\lambda_{\max}
\!\left(
A^{T}W_{P}A
\right)
+
\lambda_{\mathrm{U}}
}{
\lambda_{\mathrm{L}}
}.
\end{equation}
\end{theorem}
\begin{proof}
The lower bound follows from
\[
A^{T}W_{P}A \succeq 0
\]
and
\[
\Lambda^{(k)}\succeq\lambda_{\mathrm{L}}I,
\]
while the upper bound follows from Weyl's inequality. Since
\[
\lambda_{\mathrm{L}}>0,
\]
the coefficient matrix
\[
A^{T}W_{P}A+\Lambda^{(k)}
\]
is symmetric positive definite, implying existence and uniqueness of the MAP solution.
\end{proof}

In sparse-view, limited-angle, or otherwise rank-deficient CT, $\lambda_{\min}(A^T\WP A)$ is typically zero or nearly zero. The theorem makes explicit that the positive lower bound of $\Lambda^{(k)}$ is what regularizes poorly observed or null-space directions. This is also the relevant conditioning guarantee for CG.

\textbf{Corollary 1 (Properties of the bounded rational realization).}
For
\[
g(u)
=
\lambda_{\min}
+
\frac{\lambda_{\max}-\lambda_{\min}}
{1+\kappa u},
\qquad u\geq 0,
\]
where
\[
0<\lambda_{\min}<\lambda_{\max},
\qquad
\kappa>0,
\]
we have
\[
\lambda_{\min}<g(u)\leq\lambda_{\max},
\]
and
\[
g'(u)
=
-
\frac{
\kappa(\lambda_{\max}-\lambda_{\min})
}{
(1+\kappa u)^2
}
<0.
\]
Therefore, the map satisfies positivity, boundedness,
continuity, monotone decrease,
recovery of the maximum prior precision
$\lambda_{\max}I$
at zero epistemic uncertainty,
and convergence toward the positive precision floor
$\lambda_{\min}I$
as uncertainty increases.

For a diagonal precision $\Lambda=\diag(\lambda_r)$, the prior term becomes
\begin{equation}
\frac12\sum_r\lambda_r(x_r-\gamma_r)^2,
\end{equation}
which is a spatially adaptive quadratic regularizer. The distinction from handcrafted adaptive regularization is that the weights have a probabilistic interpretation and are estimated from learned prior reliability.

\paragraph{Local fixed-point interpretation.}
Let $F$ denote the network plus uncertainty-to-precision map and $M$ the exact MAP solution operator, so $T=M\circ F$. If $F$ and $M$ are locally Lipschitz with constants $L_F$ and $L_M$, then $L_T\le L_ML_F$. A sufficient local convergence condition is $L_ML_F<1$. This is not claimed for an unconstrained network. It indicates how spectral normalization, Lipschitz control, or contractive architecture design can support convergence. An optional averaged update,
\begin{equation}
x^{(k+1)}=(1-\eta)x^{(k)}+\eta\widetilde x^{(k+1)},\qquad 0<\eta\le1,
\end{equation}
may improve empirical stability, although averaging alone does not guarantee contraction when the underlying map is expansive.

\subsection{Transition to mechanism validation}
The preceding formulation motivates a direct mechanism-level question: does an uncertainty field contain enough spatial information to regulate prior precision usefully before a trained evidential network is available? Section~4.1 addresses this question through a small patient-held-out pilot study using transparent surrogate estimators. The pilot is intentionally limited to validating the uncertainty--precision feedback mechanism; it is not presented as a validation of NIG evidential regression or as a clinical-performance benchmark. Section~4.2 then replaces the surrogates with a trained Normal--Inverse--Gamma evidential network and evaluates the complete closed loop on held-out patients.

\section{Simulation Studies}
\subsection{Mechanism-Validation Pilot Study}
The purpose of this pilot was to test the core operational loop of ELECTRIC before undertaking full GPU-based evidential-network training.
The experiment was deliberately designed as a mechanism-level sanity check. It does not evaluate NIG evidential regression itself, and it does not claim equivalence to a carefully tuned fixed-prior reconstruction. Instead, it asks whether a data-derived uncertainty field contains enough error-predictive information to support selective trust and whether the adaptive precision mechanism behaves as intended on real CT images.

\subsubsection{Experimental setup}
A pool of 100 slices was randomly sampled from all 10 subjects in the AAPM Mayo Clinic Low-Dose CT Grand Challenge dataset~\cite{mccollough2017lowdose}. Each slice was downsampled to $128\times128$, and the routine-dose image was used as the reference $x_{\mathrm{gt}}$. Because the released image pairs do not provide raw projection measurements, projections were simulated following common benchmark practice, analogous to the construction of LoDoPaB-CT~\cite{leuschner2021lodopab}. A two-dimensional parallel-beam geometry with 120 projection angles and 185 detector elements was used. The forward operator was constructed explicitly as a sparse matrix, and its adjoint consistency was verified numerically: the relative discrepancy between $\langle Ax,y\rangle$ and $\langle x,A^Ty\rangle$ was $4.6\times10^{-7}$.

Low-dose measurements were simulated with an incident photon count $I_0=10^4$ and Poisson quantum noise. After logarithmic transformation, the measurement precision was approximated by the diagonal matrix
\[
\WP=\diag(\text{detected counts}).
\]
To isolate the feedback mechanism without neural-network training, two transparent surrogate estimators were used: (a) The image proposal $\gamma$ was obtained by total-variation denoising of the current iterate, serving as a placeholder for a learned prior mean; and (b) The uncertainty field $u$ was estimated as the voxel-wise variance of eight FBP reconstructions generated from independent noise resamplings, serving as a generic uncertainty surrogate.

The second surrogate is measurement-driven and therefore should not be interpreted as the final ELECTRIC epistemic uncertainty $u_{\mathrm{epi}}$. The pilot tests the generic uncertainty-to-precision control mechanism; full validation of the epistemic-only design requires the trained NIG network described in the preceding section.

For the mechanism-validation pilot, we retained the simpler
floor-free rational mapping used in the original implementation,
\begin{equation}
\Lambda
=
\operatorname{diag}
\left(
\frac{\lambda_{0,\mathrm{pilot}}}
{
1+\kappa u/\operatorname{median}(u)
}
\right),
\label{eq:pilot_precision}
\end{equation}
where normalization by the median makes the uncertainty scale
dimensionless. Here, $\lambda_{0,\mathrm{pilot}}>0$ denotes the
nominal maximum precision used specifically in the pilot experiment.
This pilot-specific mapping differs from the bounded formulation in
Eq.~(22), which includes a strictly positive precision floor and is
preferred for the general ELECTRIC framework.

Three outer iterations were used, and each inner MAP subproblem was solved by 60 conjugate-gradient iterations. Data splitting was strictly patient-wise. Hyperparameters were selected using three slices from subject L310; the held-out test set consisted of 17 slices from subjects L333 and L506, which were not used during tuning. 
Grid search selected a tuned fixed prior $\lambda=10^5$ and adaptive
parameters
\[
\lambda_{0,\mathrm{pilot}}=3\times10^5,
\qquad
\kappa=1.
\]

\subsubsection{Prespecified mechanism criteria and results}
Three criteria were specified before examining the final test results. Figure~\ref{fig:pilot_quant} summarizes the quantitative findings. The first panel shows the uncertainty--error relationship and the distribution of slice-wise Spearman coefficients. The second shows the risk--coverage curve. The third compares reconstruction error under tuned and deliberately mis-set prior strengths.

\paragraph{Uncertainty--error association.}
Across the 17 test slices, the median voxel-wise Spearman correlation between the uncertainty surrogate $u$ and absolute reconstruction error was $0.427$, with a range of $0.158$--$0.527$. This exceeded the prespecified median threshold of $0.4$, although the weaker correlations in some slices indicate substantial room for improvement by a learned evidential estimator.

\paragraph{Risk--coverage monotonicity.}
Voxels were ranked by uncertainty, and progressively smaller subsets containing the lowest-uncertainty voxels were retained. The mean error of the retained voxels decreased monotonically as coverage decreased. The aggregate curve was strictly monotone, and 76\% of individual slices were pointwise monotone. Thus, the uncertainty field supported selective trust: lower predicted uncertainty corresponded to lower reconstruction risk.

\paragraph{Comparison with a tuned fixed prior.}
The adaptive reconstruction achieved a mean RMSE of $3.82\times10^{-3}$, compared with $3.73\times10^{-3}$ for the tuned fixed-prior MAP baseline, a relative difference of approximately 2.2\%. The adaptive method achieved a slightly higher mean SSIM, $0.867$ versus $0.861$. These results indicate approximate parity with a carefully tuned fixed prior rather than a performance advantage.

\begin{figure}[t]
\centering
\includegraphics[width=\textwidth]{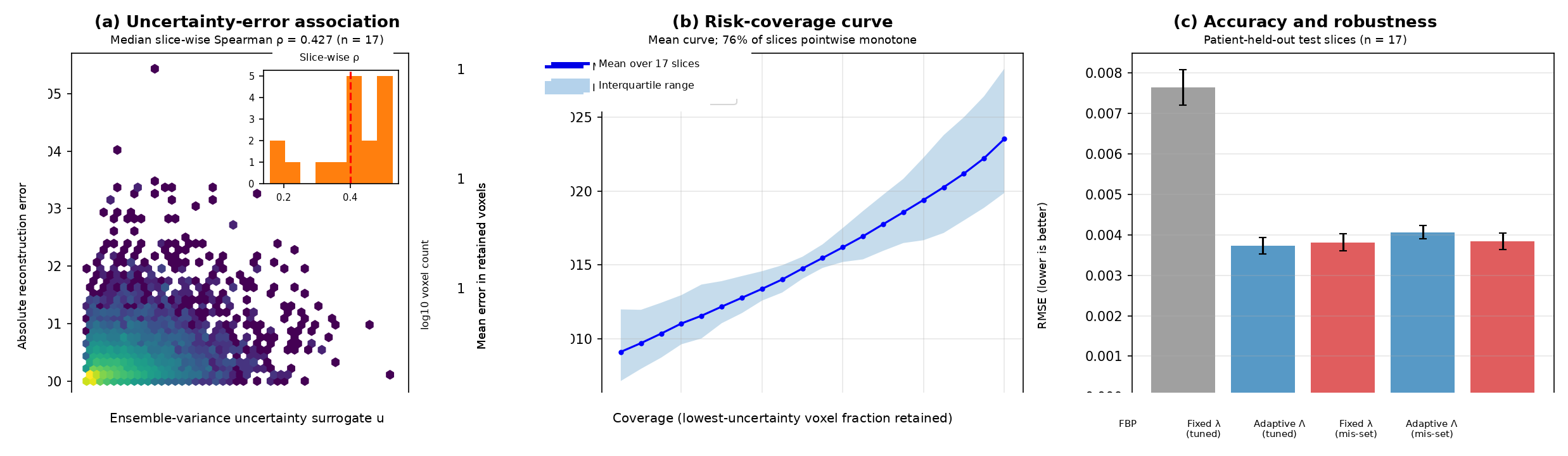}
\caption{Quantitative results of the mechanism-validation pilot study on 17 patient-held-out test slices. (a) Two-dimensional density relationship between the uncertainty surrogate $u$ and absolute reconstruction error for a representative slice; the inset shows the distribution of slice-wise Spearman coefficients, with median $\rho=0.427$. (b) Risk--coverage curve obtained by retaining progressively lower-uncertainty voxels; the mean retained-set error decreases as coverage decreases, supporting selective trust. (c) Mean RMSE and standard error. Adaptive precision is comparable with a tuned fixed prior, but is markedly more robust when the nominal prior strength is deliberately set too high.}
\label{fig:pilot_quant}
\end{figure}

\subsubsection{Robustness to prior-strength mismatch}
An additional observation was robustness to misspecification of the nominal prior strength. When the prior strength was deliberately increased to $10^6$, the fixed-prior reconstruction exhibited an RMSE degradation of approximately 8.8\% relative to its tuned setting. In contrast, the adaptive method degraded by only approximately 0.7\%. The mechanism is consistent with the intended interpretation of ELECTRIC: regions with larger uncertainty automatically receive lower prior precision, allowing the measurements to compensate for an overly aggressive global prior.
This result suggests a value proposition distinct from only improving accuracy. The adaptive precision field may reduce the need for scenario-specific tuning and provide self-protection against prior-strength mismatch. Figure~\ref{fig:pilot_visual} illustrates a representative slice. Spatially elevated uncertainty aligns qualitatively with regions of elevated adaptive-reconstruction error, consistent with the quantitative correlation and risk--coverage analyses.

\begin{figure}[t]
\centering
\includegraphics[width=\textwidth]{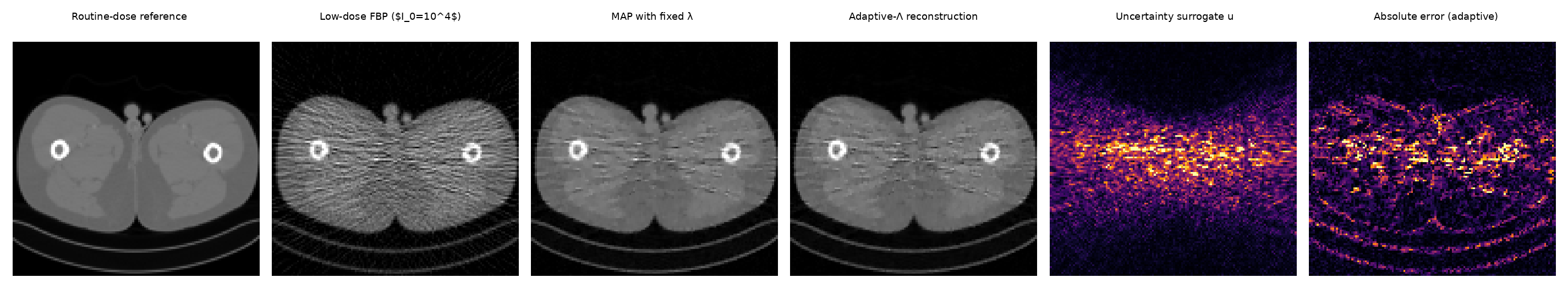}
\caption{Representative-slice reconstruction comparison and spatial correspondence between uncertainty and error. From left to right: routine-dose reference, low-dose FBP reconstruction at $I_0=10^4$, MAP reconstruction with tuned fixed precision, adaptive-precision reconstruction, uncertainty surrogate $u$, and absolute error of the adaptive reconstruction. Elevated uncertainty is spatially associated with elevated error, providing qualitative support for the quantitative analyses in Fig.~\ref{fig:pilot_quant}.}
\label{fig:pilot_visual}
\end{figure}

\subsubsection{Interpretation and limitations}
Our pilot results confirm that the ELECTRIC closed-loop mechanism is computationally executable on real Mayo CT images and that even a simple data-derived uncertainty surrogate contains useful error-predictive information. Nevertheless, the limitations are substantial and define the proper interpretation of the study.

First, both the uncertainty field and image proposal were produced by non-neural surrogate estimators. The results therefore do not validate NIG evidential regression, nor do they establish that a trained network will produce calibrated epistemic uncertainty. Second, the median correlation of approximately 0.43 is meaningful but moderate, and the low-correlation slices underscore the need for a learned estimator. Third, the $128\times128$ resolution, two-dimensional parallel-beam geometry, and small patient-held-out test set support only a mechanism-level conclusion, not a clinical-performance claim. Finally, because the uncertainty surrogate was generated from repeated measurement-noise realizations, it contains predominantly aleatoric information; the experiment therefore validates generic adaptive trust control rather than the final epistemic-only decomposition proposed for ELECTRIC.

Accordingly, this pilot is best viewed as a mechanism illustration preceding the full feasibility study. It provides an implementation skeleton and a prespecified evaluation protocol for replacing the surrogate estimators with a true NIG evidential network and testing the method at a much larger scale and under  realistic acquisition geometries.

\subsection{Feasibility Study with a Trained Evidential Network}
The mechanism-validation pilot of Section~4.1 deliberately replaced the evidential network with transparent non-neural surrogates, and it therefore left open the central question raised in Section~4.1.4: does a genuinely \emph{trained} NIG evidential network produce (i) a learned prior mean that improves on its input, and (ii) an epistemic-uncertainty field that is error-predictive enough to drive the adaptive-precision loop? This subsection reports a controlled feasibility study that answers both questions affirmatively at a small scale that nonetheless exercises the complete closed loop, using the direct evidential realization of Eqs.~(16)--(17), the bounded uncertainty-to-precision map of Eq.~(22), and the physics-guided MAP update of Eq.~(29). The study is intentionally scoped as a feasibility demonstration rather than a clinical benchmark; the limitations are stated explicitly at the end of this study.

\subsubsection{Data and imaging model}
One hundred slices from the ten patients of the AAPM Mayo Clinic Low-Dose CT Grand Challenge dataset~\cite{mccollough2017lowdose} were used. In contrast to the downsampling used in the Section~4.1 pilot, each $512\times512$ slice was reduced to $128\times128$ by extracting a randomly located $128\times128$ subregion with at least $35\%$ anatomical content; the routine-dose crop served as the reference $x_{\mathrm{gt}}$. Data were split strictly patient-wise into a training set (seven patients, 72 slices), a validation patient for hyperparameter selection (L310, 11 slices), and a held-out test set (L333 and L506, 17 slices).

Projections were simulated following common benchmark practice analogous to LoDoPaB-CT~\cite{leuschner2021lodopab}. A two-dimensional parallel-beam operator with 120 views and 185 detector elements was assembled as an explicit sparse matrix whose transpose is the exact matched back-projector; the adjoint discrepancy between $\langle Ax,y\rangle$ and $\langle x,A^Ty\rangle$ was $1.6\times10^{-16}$. Low-dose measurements used an incident count $I_0=10^4$ with Poisson noise, and after logarithmic transformation the measurement precision was the diagonal $\WP=\diag(\text{detected counts})$, exactly as in Eq.~(2).

\subsubsection{Evidential network and training}
The evidential network $f_\theta$ of Eq.~(16) was a compact U-Net ($\approx\!2.6\times10^{5}$ parameters) mapping a single-channel image estimate to the four per-pixel NIG maps $(\gamma,\nu,\alpha,\beta)$. Positivity and the moment constraints $\nu>0$, $\alpha>1$, $\beta>0$ (Table~1) were enforced with softplus parameterizations, $\alpha=\operatorname{softplus}(\cdot)+1+\varepsilon$. Supervision used the Student-$t$ NIG negative log-likelihood of Eq.~(27) and the evidence penalty of Eq.~(28), together with a light $\ell_1$ anchor on $\gamma$; the epistemic surrogate was read out through Eq.~(17).

Two practical measures were adopted, both compatible with the ELECTRIC formulation. First, in the spirit of the inference-matched training of Section~3.7, the training inputs comprised both filtered back-projections and one-step MAP reconstructions of the same crops, so that $f_\theta$ is exposed to the $k{=}0$ and $k{\ge}1$ input statistics encountered during the loop. Second, because a single evidential head produced an unstable epistemic pattern across training epochs---a known characteristic of deep evidential regression~\cite{meinert2023unreasonable}---the deployed estimates were stabilized with a snapshot ensemble~\cite{huang2017snapshot} of eight late-epoch checkpoints taken from the single training run. For each checkpoint we evaluated the predicted prior mean $\gamma$ and the epistemic surrogate $\uepi=\beta/(\nu(\alpha-1))$ of Eq.~(17), and averaged these two derived quantities across checkpoints; we did not average the raw $(\nu,\alpha,\beta)$ parameters, since a plain parameter average is not equivalent to the mixture predictive distribution. This snapshot ensemble is a low-cost variance-reduction device from a single optimization trajectory: its members are correlated and it therefore does not provide the independence of a true deep ensemble~\cite{lakshminarayanan2017ensembles}, which remains a stronger Stage-I estimator admissible within the framework.

\subsubsection{Closed-loop reconstruction and baselines}
Inference initialized $x^{(0)}=\mathrm{FBP}(p)$ and ran $K=3$ outer iterations of the loop $(\gamma,\uepi)\!\to\!\Lambda\!\to\!x^{(k+1)}$, with the bounded rational map of Eq.~(22) and each MAP subproblem solved by 60 warm-started conjugate-gradient iterations (Eq.~(29)). The adaptive configuration $(\lambda_{\max},\lambda_{\min},\kappa)$ and the tuned fixed prior $\lambda$ were selected on the validation patient alone. Four reconstructions were compared on the test set: low-dose FBP; the learned prior mean $\gamma$ alone (network output without the physics update); a tuned fixed-$\lambda$ MAP baseline sharing the same learned $\gamma$; and the full adaptive ELECTRIC reconstruction. Figure~\ref{fig:feas_quant} summarizes the quantitative findings, and Figure~\ref{fig:feas_visual} shows a representative slice.

\subsubsection{Results}

\paragraph{Learned prior mean.}
The trained network produced a strong image prior: on the held-out test slices, $\gamma$ reduced the reconstruction RMSE from $4.97\times10^{-2}$ (FBP) to $1.48\times10^{-2}$, an approximately $70\%$ reduction, confirming that the evidential proposal head learns a genuine denoising prior rather than merely reproducing its input.

\paragraph{Epistemic uncertainty and selective trust.}
The learned $\uepi$ was error-predictive. Across the 17 test slices the median voxel-wise Spearman correlation between $\uepi$ and absolute reconstruction error was $0.316$ (range $0.215$--$0.403$). For a matched comparison, the measurement-noise surrogate of Section~4.1 was re-evaluated under the identical cropping protocol used here; it yielded a lower median correlation ($\rho\approx0.19$) than the $0.427$ reported for downsampled slices, because cropping preserves high-frequency edge error that a purely measurement-driven surrogate cannot predict. The learned estimator therefore surpasses the matched surrogate baseline, the expected ordering. Ranking voxels by $\uepi$ produced a strictly monotone aggregate risk--coverage curve (Figure~\ref{fig:feas_quant}b); retaining the $20\%$ lowest-uncertainty voxels reduced the mean error to $0.58\times$ the full-image value, demonstrating that the learned uncertainty supports selective trust.

\paragraph{Closed-loop reconstruction.}
Table~\ref{tab:feas_accuracy} summarizes accuracy and data consistency. The learned prior mean already yields a high SSIM ($0.889$), but as a pure network output it is not constrained to the measurements and leaves a large weighted data-fit residual ($\chi^2/M=1.29$). The physics-guided MAP update restores measurement consistency ($\chi^2/M=0.90$ for the adaptive reconstruction, close to the ideal value of one) while attaining the lowest RMSE ($1.46\times10^{-2}$) and highest SSIM ($0.897$). This isolates the contribution of the physics step: relative to the pure network output, its principal benefit is data consistency, accompanied by a small accuracy gain. The adaptive reconstruction also slightly surpassed the validation-tuned fixed-$\lambda$ MAP baseline ($1.56\times10^{-2}$, SSIM $0.879$) that used the same learned prior mean. In a slice-level descriptive paired analysis the improvement was consistent: adaptive achieved lower RMSE on all $17$ slices, with a mean per-slice difference of $0.9\times10^{-3}$, and the advantage held within each of the two held-out patients. Because only two patients were held out, the slices are not statistically independent; we therefore treat this as descriptive evidence and defer formal significance testing to a larger patient cohort. Consistent with the conservative interpretation of Section~4.1, the adaptive mechanism is best understood as achieving parity-to-modest-improvement over a validation-tuned fixed prior rather than a large accuracy gain.

\begin{table}[H]
\centering
\caption{Reconstruction accuracy and data consistency on the 17 patient-held-out test slices (means over slices; RMSE in the normalized attenuation domain). Data consistency is the reduced weighted residual $\chi^2/M=(Ax-p)^{\top}\WP(Ax-p)/M$---equivalently, twice the data-likelihood term of Eq.~(2) divided by the number of measurements $M$---for which a value near one indicates consistency with the Poisson noise model; it is omitted for FBP. A slice-level paired comparison of the adaptive versus fixed prior is reported in the text.}
\label{tab:feas_accuracy}
\begin{tabular}{lccc}
\toprule
Method & RMSE ($\times10^{-2}$) & SSIM & $\chi^2/M$ \\
\midrule
Low-dose FBP                              & $4.97$          & $0.467$          & --     \\
Learned prior mean $\gamma$ (no physics)  & $1.48$          & $0.889$          & $1.29$ \\
MAP, validation-tuned fixed $\lambda$     & $1.56$          & $0.879$          & $0.86$ \\
Adaptive $\Lambda$ (ELECTRIC)             & $\mathbf{1.46}$ & $\mathbf{0.897}$ & $0.90$ \\
\bottomrule
\end{tabular}
\end{table}

\begin{figure}[H]
\centering
\includegraphics[width=\textwidth]{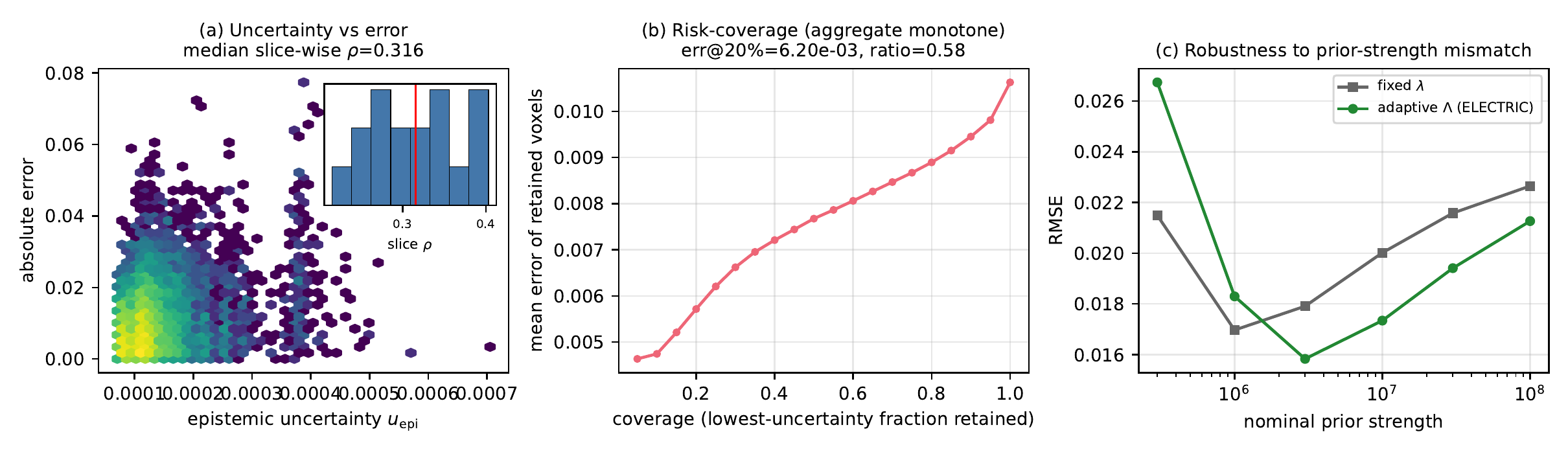}
\caption{Quantitative feasibility results on the 17 patient-held-out test slices. (a) Foreground uncertainty--error density for a representative slice; the inset shows the distribution of slice-wise Spearman coefficients, with median $\rho=0.316$. (b) Aggregate risk--coverage curve; the mean error of retained voxels decreases monotonically as lower-uncertainty voxels are retained. (c) RMSE versus nominal prior strength; the adaptive precision field attains the lowest minimum and remains below the tuned fixed prior throughout the over-regularization regime.}
\label{fig:feas_quant}
\end{figure}

\paragraph{Robustness to prior-strength mismatch.}
The distinctive behavior appears when the nominal prior strength is misspecified (Figure~\ref{fig:feas_quant}c). The tuned fixed prior is optimal only in a narrow band and degrades on either side; the adaptive field attains a lower minimum and, throughout the entire over-regularization regime ($\lambda_{\max}\!\ge\!3\times10^{6}$), remains consistently below the fixed-prior curve---for example $1.73$ versus $2.00\times10^{-2}$ at $10^{7}$ and $2.13$ versus $2.27\times10^{-2}$ at $10^{8}$. Because high-uncertainty regions automatically receive lower precision, an overly aggressive global prior is partially self-corrected by the measurements. This is precisely the value proposition anticipated in Section~4.1.3: reduced sensitivity to scenario-specific tuning in the practically relevant direction, since a strong learned prior is rarely deliberately under-weighted.

\begin{figure}[t]
\centering
\includegraphics[width=\textwidth]{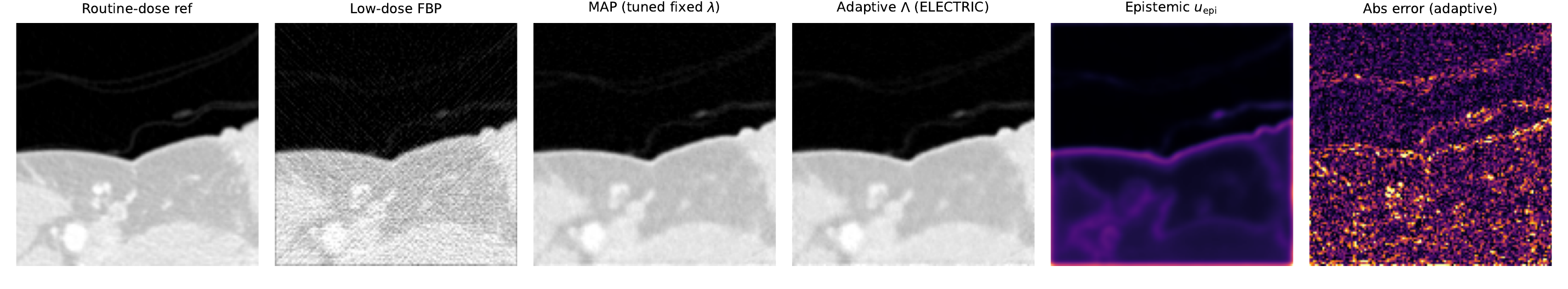}
\caption{Representative held-out slice. From left to right: routine-dose reference; low-dose FBP at $I_0=10^4$; MAP with a tuned fixed prior; adaptive-precision reconstruction; the learned epistemic uncertainty $\uepi$; and the absolute error of the adaptive reconstruction. The learned uncertainty concentrates on anatomical boundaries, and elevated uncertainty aligns spatially with elevated error, consistent with the quantitative analyses in Figure~\ref{fig:feas_quant}.}
\label{fig:feas_visual}
\end{figure}

\subsubsection{Interpretation and limitations}
This study advances beyond the Section~4.1 pilot in the decisive respect: the prior mean and the uncertainty are produced by a trained NIG evidential network rather than by surrogates, and the learned epistemic field both improves on the surrogate and drives an adaptive loop that matches or exceeds a validation-tuned fixed prior while being markedly more robust to prior-strength misspecification. Several limitations bound the claim. First, the scale is deliberately small---$128\times128$ crops, two-dimensional parallel-beam geometry, a compact network trained on a single CPU, and 17 held-out slices---so the results support a feasibility conclusion rather than a clinical-performance claim. Second, the uncertainty--error correlation, while exceeding the matched surrogate, is moderate; part of this reflects the cropping protocol, which preserves high-frequency edge error that is intrinsically harder to predict than the smoother error of downsampled images, and part reflects the known calibration difficulty of deep evidential regression~\cite{amini2020deep,meinert2023unreasonable}, which motivated the snapshot-ensemble stabilization. Moreover, the uncertainty was validated as error-predictive---through uncertainty--error correlation and risk--coverage---rather than as formally calibrated; reliability diagrams and proper scoring rules were not assessed and are left to future work. Third, the evidential network was trained stage-wise on the prior and uncertainty objectives rather than jointly through the final reconstruction loss, and the uncertainty-to-precision map was the fixed bounded form of Eq.~(22); joint training through the reconstruction objective, the learnable mapping $g_\phi$ of Section~3.6, and the Bayesian-consistent virtual-evidence parameterization of Section~3.5 remain natural next steps. Taken together, the experiment provides a complete closed-loop demonstration of the ELECTRIC pipeline with a trained evidential network and a prespecified protocol for scaling to fan/cone-beam geometries, full-resolution images, and larger cohorts.

\section{Discussion}
%\section{Unified Interpretation and Practical Implications}

The two simulation studies in Section~4 establish the empirical status of ELECTRIC. The mechanism-validation pilot (Section~4.1) shows, with transparent surrogates, that an uncertainty field carries enough error-predictive information to steer prior precision. The feasibility study (Section~4.2) then closes the loop with a trained Normal--Inverse--Gamma evidential network: the learned prior mean improves substantially on filtered back-projection, the learned epistemic uncertainty is error-predictive and yields a strictly monotone aggregate risk--coverage relationship, and the physics-guided MAP update restores the measurement consistency that the pure network output lacks while matching or exceeding a validation-tuned fixed prior and improving robustness to prior-strength misspecification. This constitutes a complete closed-loop demonstration of the pipeline. It remains a feasibility-scale result: the network is trained stage-wise rather than jointly through the final reconstruction loss, the uncertainty-to-precision map is fixed rather than learned, and formal calibration of the predictive uncertainty is not yet established. These considerations shape the directions discussed below.

The proposed framework distinguishes three conceptually different layers
that are often conflated in uncertainty-aware image reconstruction.

First, the probabilistic model specifies a statistical representation of
uncertainty.
In the present implementation, this representation is provided by the
Normal--Inverse--Gamma (NIG) family, which yields analytically convenient
decompositions of aleatoric and epistemic uncertainty.
However, the NIG model should be regarded as a modeling assumption rather
than an absolute description of the underlying uncertainty.
Its principal role is to provide a structured and differentiable latent
representation that can be jointly optimized with image reconstruction.

Second, deep learning performs a data-driven adaptation of this statistical
representation.
The evidential network jointly predicts an image proposal together with its
associated uncertainty parameters under common supervision.
Consequently, the learned uncertainty is no longer merely the analytical
uncertainty implied by the original NIG model, but rather an empirical
representation optimized for the reconstruction task.
This constitutes a data-driven adaptation of the statistical uncertainty representation.

Third, uncertainty itself is not the final objective.
Instead, the uncertainty representation is further transformed into a
positive precision map that determines how strongly the learned image
proposal should influence the subsequent physics-guided MAP reconstruction.
Because this transformation is optimized directly through the final
reconstruction objective, the resulting precision map should be interpreted
as task-aware prior confidence rather than as statistical uncertainty.
Accordingly, uncertainty estimation becomes an intermediate representation,
while adaptive prior confidence becomes the quantity that directly controls
image formation.

This viewpoint also clarifies the relationship between the proposed
framework and existing evidential learning methods.
Current deep evidential regression typically predicts the four
Normal--Inverse--Gamma parameters directly.
Such a parameterization is computationally convenient and has demonstrated
good empirical performance.
Nevertheless, the predicted parameters should not necessarily be interpreted
as exact Bayesian posteriors arising from conjugate updating.
Recent theoretical analyses suggest that only a restricted subset of the
NIG parameter space is reachable through exact Bayesian updating from a
fixed prior.
This observation motivates future Bayesian-consistent parameterizations in
which the network predicts virtual sufficient statistics, such as an
effective evidence size, sample mean, and sample scatter, from which the
NIG parameters are obtained analytically through conjugate updating.
Such realizations naturally satisfy Bayesian structural constraints while
remaining fully compatible with the two-stage ELECTRIC framework.
Importantly, the present framework is independent of the specific
parameterization used for uncertainty estimation.

Several practical considerations remain important.
First, the learned task-aware prior confidence should not be interpreted as
a calibrated posterior precision, since it may additionally absorb model
mismatch, optimization behavior, and task-specific adaptation.
Second, Stage-I uncertainty should be evaluated using uncertainty metrics
such as calibration, proper scoring rules, reliability diagrams, and
out-of-distribution detection, whereas Stage-II should be evaluated by its
ability to improve reconstruction quality, data consistency, robustness,
and numerical stability.
Third, identifiability deserves further investigation because different
parameterizations of uncertainty or precision may produce similar
reconstruction performance.
Finally, learnable uncertainty-to-precision mappings, including neural
networks or other differentiable monotone transformations, provide a
natural extension of the present framework and may further improve adaptive
precision estimation for reconstruction.

The proposed framework is not restricted to the NIG family or to computed
tomography.
Stage-I uncertainty may arise from evidential learning, Bayesian neural
networks, variational inference, diffusion-based uncertainty estimation,
ensemble methods, or future probabilistic representations.
Likewise, Stage-II adaptive precision can naturally be integrated into MRI,
PET, SPECT, ultrasound, multimodal reconstruction, or other inverse
problems.
The central principle therefore remains unchanged: statistical
uncertainty describes what is learned from data, whereas task-aware
Bayesian prior confidence determines how strongly that learned
knowledge should participate in physics-guided image reconstruction.
From this perspective, adaptive Bayesian reconstruction becomes a
closed-loop decision process in which statistical learning determines
not only what image hypothesis should be generated, but also how
the associated confidence level should influence the final reconstruction.

\end{document}